\documentclass{llncs}
\usepackage{amssymb,color,amsmath,hyperref,mathtools}
\usepackage{amssymb,color,amsmath,mathtools}
\usepackage{algorithmic}

\newlength\problemlength
\newcommand\dproblem[3]{%
\begin{list}{}{\labelwidth\problemlength \labelsep.7em \rightmargin1.5em
\leftmargin\problemlength \advance\leftmargin by3em%2em
\parsep0ex \itemsep.2ex plus.1ex}
\item[{\sl Problem :\hfill}] #1
\item[{\sl Instance :  \hfill}] #2
\item[{\sl Question : \hfill}] #3
\end{list}
}

\newcommand\revb[1]{\ast_{#1}} %révisions syntaxiques
\newcommand\crevb[1]{\circledast_{#1}}%%%%%%%%%%%%%%%%%%%%%%%%%%%%%%%

\newcommand\nmodels{\not\models}

\newcommand{\calL}{{\mathcal{L}}}

\newcommand{\V}{{\mathcal V}}

\usepackage{xspace}
\newcommand\mSigmap[1]{$\sigmap{}$}

\newcommand{\rsrg}{\mathrm{RSRG}}
\newcommand{\csrg}{\mathrm{CSRG}}

\newcommand{\ginsberg}{\mathrm{G}}
\newcommand{\widtio}{\mathrm{wid}}

\newcommand{\rsrw}{\mathrm{RSRW}}
\newcommand{\csrw}{\mathrm{CSRW}}
\newcommand{\csir}{\mathrm{CSIR}}

\newcommand\Wcal{\mathcal{W}}
\newcommand\Wcali{\Wcal_\subseteq}

\newcommand\Wcalc{\Wcal_\mathit{card}}

\newcommand\Wcalbl{\Wcal_\mathit{Bel}}
\newcommand\Vcalbl{\Vcal_\mathit{Bel}}
\newcommand\calW{\mathcal{W}}
\newcommand\Vcal{\mathcal{V}}

\newcommand\Foc{\mathcal{F}}
\newcommand\Foci{\Foc_\cap}

\newcommand\card[1]{|#1|}

\newcommand\ie{i.e.\xspace}
\newcommand\wrt{w.r.t.\xspace}

\title{On the use of evidence theory in  belief base revision 
}
\titlerunning{On the use of evidence theory in  belief base revision  }

\author{ Ra\"\i da~Ktari\inst{1,3} \and Mohamed Ayman Boujelben\inst{2,3}}
\institute{Institut Supérieur d'Informatique et de Multim\'edia, Universit\'e de Sfax,
Route de Tunis Km 10, Technopole de Sfax B.P 242-3021, Sfax, Tunisie.
\and Institut des Hautes Etudes Commerciales, Universit\'e de Sfax, Route Sidi Mansour Km 10 B.P 43-3061, Sfax, Tunisie.
\and Laboratoire OLID (LR19ES21), Institut Sup\'erieur de Gestion Industrielle, Universit\'e de Sfax,
Route de Tunis, Technopole de Sfax B.P 1164-3021, Sfax, Tunisie.
\texttt{raida.ktari@isims.usf.tn }
\texttt{ ayman.boujelben@ihecs.usf.tn}
}
\authorrunning{ R.~Ktari,   and M.A.~Boujelben}

\begin{document}
\maketitle
\begin{abstract} 
This paper deals with belief base revision that is a form of belief change 
consisting of the incorporation of new facts into an agent's beliefs represented by a finite set of propositional formulas.
In the aim to guarantee more reliability and rationality for real applications while performing revision,
we propose the idea of credible belief base revision yielding to define two new formula-based
revision operators using the suitable tools offered by evidence theory.
These operators, uniformly presented in the same spirit of others in \cite{CreignouKP17},
stem from consistent subbases maximal with respect to credibility instead of set inclusion and cardinality. 
Moreover, in between these two extremes operators, evidence theory let us shed some light on a \textit{compromise} operator
avoiding losing initial beliefs to the maximum extent possible.
Its idea captures maximal consistent sets stemming from all possible intersections of maximal
consistent subbases. 
An illustration of all these operators and a comparison with others are inverstigated by examples.

\begin{keywords}
 belief base revision $\cdot$ evidence theory $\cdot$  credibility $\cdot$ rationality $\cdot$
 knowledge representation and reasoning.
\end{keywords}

\end{abstract}
%%%%%%%%%%%%%%%%%%%%%%%%%%%%%%%%%%%%%%%%%%%%%%%%%%%%%%%%
\section{Introduction}
\label{sec:intro}
%%%%%%%%%%%%%%%%%%%%%%%%%%%%%%%%%%%%%%%%%%%%%%%%%%%%%%%%
One of the important research topics in Artificial Intelligence  is dynamics
(or change) beliefs. In many applications, such as image processing, 
reliability expert opinions, robotics, radar detection and relational databases, 
intelligent agents face incomplete, uncertain and inaccurate information, and often need a revision
operation so as to manage their beliefs change in presence of a new and reliable information.
When this information contradicts the agent's current beliefs, the revision
then deals with remaining consistency in order to integrate the new information while
modifying the initial beliefs as little as possible.
During the last thirty years, the question of how to perform
revision gave rise to numerous works according to the representation of beliefs.

The first works on the revision of beliefs come from subjective probabilities,
mainly with the works of Richard Jeffrey \cite{Jeffrey83}. In this context, the beliefs of an agent are
represented by a measure of probability, and the beliefs revision is what Jeffrey
called \textit{probabilistic kinematics}. 
Shortly thereafter, many logical approaches to revision have been developed.
Among these approaches, there are the so-called \textit{syntactic approach} 
\cite{Gin86,Winslett89,Pap92,Han98c,BBPW10}
where a great importance was given to revision of finite belief bases, \ie finite propositional formulas sets.
This approach was formalized in terms of postulates ({\em  AGM postulates}) \cite{AGM85} and several operators 
have been proposed in the literature. Some of them are based on construction of maximum 
consistent subbases according to different criteria \cite{BCDLP93,Dekleer90,Leh95}.
Recently, Creignou and colleagues in \cite{CreignouKP17} focused particularly on two operators, namely RSRG and RSRW, 
that are respectively similar to Ginsberg's one \cite{Gin86} and Widtio \cite{Winslett89},
using set cardinality instead of set inclusion as maximality criterion. 
% Then, they generalize 
% the RSRG and RSRW operators to stratified belief bases giving birth thereby to other operators,
% Prioritized RSRG (PRSRG) and Prioritized RSRW (PRSRW).

An important issue is to introduce efficient tools that fulfill the needs of others in their investigations.
Indeed, in many real applications, belief bases are quite large and choosing, as a revision result,
all maximal consistent subbases like Ginsberg's approach can be an expensive, exhaustive and even explosive solution.
On the one hand, keeping only beliefs that are not questioned (streeming from the intersection of maximal coherent subbases)
can be plausible strategy, but not always reliable and can cause, in most the cases, a lot of loss of information
since there is no guarantee to  not have the empty set as an intersection between maximal consistent subbases.
On the other hand, the selection of some maximal consistent subbases according to determined criteria,
can be shown also plausible, and the choice of set cardinality like as in \cite{CreignouKP17} seems rather reasonable,
because it respects the minimality change criterion of belief revision.

To deal with this issue and in order to be reliable and close
to the non-monotony of human reasoning mainly in the case of large belief bases, 
the selection of consistent subbases maximal with respect to set cardinality is not always a guarantee 
to select the most relevant information and can consequently neglect potential formulas from initial 
agent's beliefs (as shown in Example \ref{exp:exp_principal1}). 
This paper goes one step further in this context 
by investigating a more \textit{natural} criterion 
in the selection of consistent subbases respecting the minimality change criterion of belief revision
%, not in terms of maximal number of formulas remainded,
%and on the other hand requires but
in the sense of \textit{credibility} while performing change
by capturing the most valuable or potential information from initial beliefs.
This can be accomplished by using different tools offered by evidence theory,
commonly also known as Dempster-Shafer or belief functions theory \cite{ay1,ay2}.

% More specifically, this work sheds some light on using the different tools offered by 
% evidence theory in belief base revision. 
%This model, also known as Dempster-Shafer or belief functions theory, 
%is a convenient framework dealing with imperfect information. 
%It was initially introduced by Arthur Dempster in 1967 \cite{ay1}
%and then formalized by Glenn Shafer in 1976 \cite{ay2} as a generalization of subjective probability theory. 
%This formalism has been the starting point of other theoretical developments in particular the transferable belief model
%\cite{ay3}. In addition, it has been applied in several fields such as artificial intelligence \cite{ayar1, ayar2}, 
%clustering \cite{ayc1, ayc2}, multicriteria decision aid \cite{aymcda1, aymcda2, aymcda3}, etc.

To the best of our knowledge, although belief revision in probability theory is fully studied, belief revision strategies in
evidence theory has not been addressed so far, except the paper \cite{MaLDP11} where authors have
taken revision  from a different angle representing agent's beliefs by mass functions in 
the aim to generalize Jeffrey’s rule from probability to belief functions.
In the present paper and from another point of view to revision,
we shed light on the use of evidence theory in the context of formula-based belief revision through two main contributions. 
Following the same spirit of the work \cite{CreignouKP17}, we propose at first two formula-based operators
similar to RSRG and RSRW, namely CSRG and CSRW, based on the selection of the most credible consistent subbases.
The computation of credibility is assured using the suitable tools offered by evidence theory. 
Then, we present a new revision strategy instantiated by a compromise operator CSIR between CSRG and CSRW.
This strategy captures the most credible consistent sets stemming from all the possible intersections
of maximal consistent subbases. It can supersede other strategies in many real applications for the reason 
that it avoids losing original beliefs as far as possible. A compact representation of all these operators is given
within the unified framework already developed in \cite{CreignouKP17}. 
% We finally report on associating these operators to their corresponding families in the context of belief contraction.

This work is organized as follows: 
after a preliminary section (section \ref{sec:prel})
which introduces some preliminaries on propositional logic and evidence theory
and gives a reminder on formula-based (syntactic) revision operators we are interested in,
we formally define in Section \ref{sec:New}
our new belief operators stemming from consistent subbases maximal 
with respect to \textit{credibility} degree and we show their specificities and advantages comparing with others.
In section \ref{sec:NewRev} we provide a new (compromise) revision strategy based likewise on \textit{credibility} degree
capturing beliefs with a mind to be prudent.
Both contributions are illustrated with examples. 
% We report finally on related work in Section \ref{sec:Relatedworks}.

%%%%%%%%%%%%%%%%%%%%%%%%%%%%%%%%%%%%%%%%%%%%%%%%%%%%%%%%
\section{Preliminaries}
\label{sec:prel}
%%%%%%%%%%%%%%%%%%%%%%%%%%%%%%%%%%%%%%%%%%%%%%%%%%%%%%%%
%%%%%%%%%%%%%%%%%%%%%%%%
\subsection{Propositional Logic}
\label{subsec:propositional}
%%%%%%%%%%%%%%%%%%%%%%%%%% 
In this section, we assume familiarity with the basic notions of propositional (classical) logic.
So, we (very briefly) present the background and terminology used in this paper.
Let $\calL$ be the language of propositional logic built on an infinite countable set of variables (atoms)
 denoted by $\V$ and equipped with standard connectives $\neg$, $\wedge$, $\vee$, $\rightarrow$, 
 the exclusive or connective $\oplus$, and constants $\top$, $\bot$. 
We remind that a literal $a$ is an atom (positive literal) or the negation of an atom $\neg a$ (negative literal).
%Let $A$ be a set of atoms, $\lit(A)$  denotes the set of literals over $A$.
A clause is a disjunction of literals.
We say that a formula is in  CNF  if it is a conjunction of clauses. 
%  A formula is called \emph{Horn} if it is in CNF and in each clause at most one literal is positive.
%  A formula is called \emph{Krom} if it is in CNF and   each clause has at most two literals.
% It is convenient to identify any truth assignment with the set of variables that are true in this assignment. 
% It is thus possible to consider the cardinality of a model of a formula.
% Let $\varphi$ be a formula, we denote by $\mod(\varphi)$ the set of models of  $\varphi$.
% A formula $\psi$ is a logical  consequence of $\varphi$,  denoted by $\varphi \models \psi$, 
% if $\mod(\varphi)\subseteq \mod(\psi)$, and the two formulas are equivalent,
% denoted by $\varphi \equiv \psi$, if $\mod(\varphi)=\mod(\psi)$.
For a set $A$ of formulas, $Cn(A)$ denotes the closure of $A$ under the consequence relation $\models$. 
A theory $A$ is a deductively closed set of formulas if $A=Cn(A)$.
Let $B$ be a finite set of propositional formulas, $B=\{\varphi_1, \varphi_2,\ldots,\varphi_p\}$ 
is identified to $\bigwedge B$ the  conjunction of its formulas, $\varphi_1 \land \varphi_2 \land \ldots \land \varphi_n$.
Given a family of finite sets of formulas $\calW =\{B_1,\ldots, B_n\}$, we use 
$\displaystyle\bigvee_{i=1}^n \bigwedge B_i$ for 
%representing 
$\displaystyle\bigvee_{i=1}^n \bigwedge _{\varphi\in B_i}\varphi$.

%%%%%%%%%%%%%%%%%%%%%%%%%%%%%%%%%%%%%%%%%%%%%%%%%%%%%%%%
\subsection{Belief base revision}
\label{sec:rev}
%%%%%%%%%%%%%%%%%%%%%%%%%%%%%%%%%%%%%%%%%%%%%%%%%%%%%%%%
In this paper, we focus on formula-based (syntactic) revision operators already presented within a unified framework 
in \cite{CreignouKP17}. Each operator, denoted by $\ast$, is a function that takes a belief base $B$ 
and a formula $\mu$ representing new information as input and returns a new belief base $B\ast\mu$. 
%from which one can define a new operator, denoted by  $\circledast$.
%whose result is a theory, $B\circledast\mu = Cn(B\ast\mu)$. 
Many formula-based operators stem from  $\Wcal(B,\mu)$, the set of maximal 
%(with respect to set inclusion)
subbases of $B$ consistent with  $\mu$. 
They then make use of this set to define the revised belief base according to a given strategy. 
The maximality criterion as well as the strategy can vary.
In the literature, maximality was first considered in terms of set inclusion, and thus the following set was considered
%%%%%%%%%%%%%%%%%%%%%%%%%%%%%%%
%\ginsberg and \widtio
%%%%%%%%%%%%%%%%%%%%%%%%%%%
%\begin{small}
\begin{eqnarray*}
\Wcali(B, \mu) & = & \{ B'_i\subseteq B \mid \bigwedge B'_i\nmodels \neg \mu \hbox{ and for all } B'_j, j \neq i
   \\ \  &  & \hbox{ s. t. }  B'_i\subset B'_j \subseteq B, \bigwedge B'_j \models \neg \mu\}.
\end{eqnarray*} 
%\end{small}

In \cite{CreignouKP17}, authors consider then two different strategies.
The first one is \textit{''permissive``} by considering that all maximal subbases are equally plausible
and the second one is more  \textit{''drastic``} and  stems from the intersection of consistent maximal subbases,
\ie it only keeps beliefs that are not questioned.
Thus, these two  strategies provide two well-known operators,  namely  Ginsberg's operator, 
$\revb{\ginsberg}$,  \cite{Gin86} and  Widtio operator, $\revb{\widtio}$, \cite{Winslett89} defined respectively as 

%\begin{small}
$$B \revb{\ginsberg}\mu= \bigvee_{B' \in \Wcali(B,\mu)}\bigwedge(B' \cup \{\mu\}) $$
 and 
 $$B \revb{\widtio} \mu= \bigwedge \bigcap_{B' \in \Wcali(B,\mu)} (B' \cup \{\mu\} ).$$
%\end{small}
%%%%%%%%%%%%%%%%%%%%%%%%%%%%%%%
%\rsrg and \rsrw
%%%%%%%%%%%%%%%%%%%%%%%%%%%
Authors in \cite{CreignouKP17} focus then on  maximality defined in terms of cardinality.
This is a quite natural issue since in various applications  the cardinality criterion 
is used because information acquisition is expensive.
So, they consider the set of consistent subbases maximal \wrt cardinality $\Wcalc(B,\mu)$ instead of $\Wcali(B, \mu)$.
Formally, we present  $\Wcalc(B,\mu)$ as
%\begin{small}
\begin{eqnarray*}
\Wcalc(B,\mu)  & = & \{ B'_i \subseteq B \mid \bigwedge B'_i\nmodels \neg \mu 
\hbox{ and for all } B'_j\subseteq B, j\neq i
\\ \  &  & \hbox{ s. t. }\card{B'_i} < \card{B'_j}, \bigwedge B'_j\models \neg \mu\}.
\end{eqnarray*} 
%\end{small}

Analogously and respectively to Ginsberg's and Widtio operators, the two  strategies presented above provide two operators 
$\rsrg$ \cite{BBPW10} and $\rsrw$. Indeed, the notation RSR comes from the expression \textit{``Removed Sets Revision''} 
qualifying operators stemming from the removal of the smallest number of formulas 
from the initial belief base \cite{BBPW10}. Formally, we have 

$$B \revb{\rsrg}\mu= \bigvee_{B' \in \Wcalc(B,\mu)}\bigwedge(B' \cup \{\mu\})$$
 and 
$$B \revb{\rsrw} \mu= \bigwedge \bigcap_{B' \in \Wcalc(B,\mu)} B' \cup \{\mu\}.$$

%%%%%%%%%%%%%%%%%%%%%%%%%%%%%%%
%\prsrg and \prsrw
%%%%%%%%%%%%%%%%%%%%%%%%%%%
% Authors, in \cite{CreignouKP17}, extend $\rsrg$  and  $\rsrw$ operators to stratified belief bases.
% We remind that a stratified belief base $B=(S_1, ..., S_n)$ is provided by a partition of the belief
% base in strata $S_i$ ($1 \leq i \leq n$) representing priorities between formulas. 
% Let $X\subseteq B$ be a set of formulas, they define by a tuple of integers  
% $trace(X,B)=( \card{X\cap S_1},...,\card{X\cap S_n}).$
% The usual lexicographic order over traces, denoted by $\leq_{lex}$, 
% provides a new maximality criterion for consistent subbases. 
% 
% They define the set of consistent subbases maximal \wrt this criterion by
% $$\Wcalcl(B,\mu) = \{ B_1 \subseteq B \mid \bigwedge B_1\nmodels \neg \mu 
% \hbox{ and for all }   B_2\subseteq B  \hbox{ s.t.}  trace(B_1,B) <_{lex} trace(B_2,B), \bigwedge B_2\models \neg \mu\}.$$
% Observe that all elements of $\Wcalcl(B,\mu)$ have the same trace, denoted by $\tracemax(B,\mu)$.
% 
% Similarly to the operators presented above, the operators $\prsrg$ \cite{BBPW10,BCDLP93} and $\prsrw$ are thus defined by: 
% 
% \begin{small}
% $$B \revb{\prsrg}\mu= \bigvee_{B' \in \Wcalcl(B,\mu)}\bigwedge(B' \cup \{\mu\}) \hbox{ and }
% B \revb{\prsrw} \mu= \bigwedge \bigcap_{B' \in \Wcalcl(B,\mu)} (B' \cup \{\mu\} ).$$
% \end{small}
% 
% 
% Remark that $\Wcalcl(B,\mu) \subseteq \Wcalc(B,\mu)$.
% Hence, when a belief base is not stratified, cardinality  and  trace coincide  and the operators $\prsrg$ and $\prsrw$
% coincide with the operators $\rsrg$ and $\rsrw$ respectively. 

Moreover, note that these formula-based operators 
%\cite{Gin86, BCDLP93, Dekleer90, Lehmann95, Han98c, Winslett89, Pap92, WJP00, BBPW10}
are sensitive to the syntactic form of the knowledge representation. 
The following example illustrates this idea.
\begin{example}
Consider $B_1=\{a, b\}$, $B_2=\{a, a\rightarrow b\}$ two belief bases
and a formula $\mu= \neg b$ representing the new information. The bases $B_1$ and $B_2$ are equivalents. 
The unique subset of $B_1$ which is consistent with $\mu$ is
$\{a\}$, while there are two maximal (in terms of set inclusion) subsets of $B_2$ which are consistent with $\mu$, namely, 
$\{a\}$ and $\{a \rightarrow b\}$. 
Consequently, $B_1 \revb{wid} \mu$ = $a\land \neg b$ and $B_2 \revb{wid} \mu$ = $\neg b$.
%Therefore, $B_1 \crevb{wid} \mu\neq B_2\crevb{wid}\mu$. 
\end{example}

Therefore, in order to get rid of syntax dependency, 
Hansson \cite{Han98c} has shown that it seems natural to revise explicitly defined belief bases 
and to then extend these operations to belief sets, considering the deductive 
closure of the result of revision.
Otherwise, it is possible to define, from a revision operator $\ast$, a new one denoted by $\circledast$ whose
the result is a set of beliefs (or theory) such as 
$$ B \circledast \mu = Cn(B \ast \mu) .$$
We adopt this point of view in the present paper and we obtain correspondingly the following operators.
%%%%%%%%%%%%%%%%%%%%%%%%%%%%%%%%%%%%%%%%%%%%%%%%%%%%%%%%%%%%%%%%%%%%%%%

$$B \crevb {G} \mu = Cn(\bigvee_{{B'\in \Wcali(B,\mu)}} \bigwedge B'\cup\{\mu\})$$
%\hbox{ and }
$$B \crevb {wid} \mu = Cn(\bigwedge \bigcap_{B' \in \Wcali(B,\mu)} \{B' \cup \{\mu\} \})$$
%\end{center}
%%%%%%%%%%%%%%%%%%%%%%%%%%%%%%%%%%%%%%%%%%%%%%%%%%%%%%%%%%%%%%%%%%%%
%\begin{center}
$$B \crevb {\rsrg} \mu = Cn(\bigwedge \bigvee_{{B'\in \Wcalc(B,\mu)}} \bigwedge B'\cup\{\mu\})$$
%\hbox{ and }
%and
$$B \crevb {\rsrw} \mu= Cn(\bigcap_{B' \in \Wcalc(B,\mu)} \{B' \cup \{\mu\} \})$$
%\end{center}
%%%%%%%%%%%%%%%%%%%%%%%%%%%%%%%%%%%%%%%%%%%%%%%%%%%%%%%%%%%%%%%%%%%%
% \begin{center}
% $B \crevb {\prsrg} \mu = Cn(\bigvee_{{B'\in \calU(B,\mu)}} \bigwedge B'\cup\{\mu\}),$
% %\hbox{ and }
% $B \crevb {\prsrw} \mu= Cn(\bigwedge \bigcap_{B' \in \calU(B,\mu)} \{B' \cup \{\mu\} \}).$
% \end{center}
%An illustrative example is given as follows.

%%%%%%%%%%%%%%%%%%%%%%%%%%%%%%%%%%%%%%%%%%%%%%%%%
\subsection{Evidence theory} 
\label{Evidenceth}
%%%%%%%%%%%%%%%%%%%%%%%%%%%%%%%%%%%%%%%%%%%%%%%%%
Evidence theory has been considered as a convenient
framework dealing with imperfect information. It was initially introduced by 
Arthur Dempster in 1967 \cite{ay1} and then formalized by Glenn Shafer in 1976 \cite{ay2} 
as a generalization of subjective probability theory.
It has been the starting point 
of several theoretical developments especially the transferable belief model \cite{ay3}. 
In addition, it has been applied in several fields such as artificial intelligence \cite{ayar1,ayar2}, 
clustering \cite{ayc1,ayc2}, multicriteria decision aid \cite{aymcda3,aymcda1,aymcda2}, etc.

Let $\Theta=\{S_1,...,S_n\}$ be a finite set of mutually exclusive and exhaustive statements
called frame of discernment and $2^\Theta$ be the power set of $\Theta$.
A Basic Belief Assignment (BBA) \cite{ay2} is the basic function used in evidence theory for
modeling imperfect data. 
It is a mapping $m$ defined from $2^\Theta$ to $[0,1]$ such as $m\{\emptyset\}=0$ and $\sum_{A\subseteq\Theta} m(A)=1$. 
The quantity $m(A)$ represents the belief mass of subset $A$, i.e., the belief committed exactly to $A$. 
When $m(A)\neq 0$, $A$ is called a focal element or a focal set. 

The function $m$ constitutes a flexible tool in evidence theory that models every state of belief.
A BBA is said to be \textit{Bayesian} if all its focal elements are singletons and \textit{consonant} 
if all these elements are nested. It is called \textit{vacuous} if the total belief is assigned 
only to $\Theta$ (total ignorance case) and \textit{simple} 
if it has two focal elements and $\Theta$ is one of these focal sets. 
In the latter case, $m(\Theta)$ reflects an ignorance level since it is the belief mass which 
is not assigned to any subset $A\neq\Theta$ and transferred to $\Theta$.

A BBA can be also represented by two functions called credibility (or belief) 
and plausibility, denoted in the literature respectively by $Bel$ and $Pl$ \cite{ay2}. 
Formally, these two functions are defined from $2^\Theta$ to $[0,1]$ as follows: 
%\begin{equation}   
% Bel(A)=\sum_{X\subseteq A \atop {X \neq\emptyset}} m(X)
$$Bel(A)=\sum_{\substack{X\subseteq A \\ X \neq\emptyset}} m(X)$$

%\end{equation}
%\begin{equation}   
$$Pl(A)=\sum_{A\cap X\neq\emptyset}m(X) $$
%\end{equation}

$Bel(A)$ is the total belief of subsets $X$ which are included in $A$ whereas $Pl(A)$ 
is the total belief of subsets $X$ having  a non-empty intersection with $A$, i.e.,
the subsets that are included in $A$ and those having a partial intersection with $A$. 
$Bel(A)$ and $Pl(A)$ are therefore the minimal and maximal total beliefs committed to $A$. 
They are also connected by the relation $Pl(A)=1-Bel(\overline{A})$ where $\overline{A}$ 
is the complement of $A$ in $\Theta$.

The combination is a fundamental notion in evidence theory allowing the aggregation 
of imperfect information given by several sources and modeled by BBAs. 
Several combination rules have been developed in this context \cite{ay4,ay5,ay6}.
Among them, Dempster's rule \cite{ay2} remains the most commonly-used operator 
in the combination of independent BBAs. It is given by 
%%%%%%%%%%%%%%%%%%%%%%%%%%%%%%%%%%%%%%%%%%%%%%%%%%%
% \begin{equation}   
$$m(A)=(1-k)^{-1}.\sum_{X\cap Y=A}m_1(X).m_2(Y),$$
% \end{equation}
%%%%%%%%%%%%%%%%%%%%%%%%%%%%%%%%%%%%%%%%%%%%%%%%%%% 
where $m=m_1\oplus m_2$ is the BBA deduced from the combination of $m_1$ and $m_2$ 
(called orthogonal sum) and $k=\sum_{X\cap Y=\emptyset}m_1(X).m_2(Y)$ 
is the belief mass that the combination assigns to the empty set. 
The ratio $(1-k)^{-1}$ is a normalization factor guarantying that 
no belief mass is given to the empty set and that the total belief is equal to one.  

Dempster's rule is a conjunctive operator, i.e., the resulting focal
elements are intersections of those related to $m_1$ and $m_2$. 
It can be proved to be both commutative and associative.
Thus, the combination result of several BBAs is independent of the order in which they are considered. 

The decision-making is also an important notion of evidence theory that aims to choose 
the "best" statement of $\Theta$. Among other rules, one can cite the maximum of credibility rule 
that selects the most credible $S_i$ \cite{ay7}, the maximum of plausibility rule that chooses
the most plausible $S_i$ \cite{ay7}, 
and the maximum of pignistic probability \cite{ay8}. The latter operator is based on the idea of transforming a BBA 
into a function having similar properties of a probability distribution called pignistic probability 
function $BetP$. 
The decision is therefore to choose the statement having the maximum of pignistic probability.

%%%%%%%%%%%%%%%%%%%%%%%%%%%%%%%%
\section{Credible belief base revision} 
\label{sec:New}
%%%%%%%%%%%%%%%%%%%%%%%%%%%%%%%%
In this section, we investigate the idea of belief base revision considering maximality in terms of set credibility
(instead of set inclusion and set cardinality) denoted throughout this paper by CSR (Credible Sets Revision). 
Recall that our goal is to define belief base revision operators requiring \textit{rationality}
when revising so as to avoid losing valuable beliefs.
%%%%%%%%%%%%%%%%%%%%%%%%%%%%%%%%
\subsection{Credible belief operators} 
\label{subsec:NewCred}
%%%%%%%%%%%%%%%%%%%%%%%%%%%%%%%%
As described hereafter, the \textit{credible belief base revision} leads to define two new formula-based revision operators 
using the suitable tools offered by evidence theory.
To ensure uniformity with the RSRG and RSRW operators, we denote these operators by 
CSRG (referring to the \textit{permissive} strategy) and CSRW (referring to the \textit{drastic} strategy)
that stem from $\Wcalbl(B,\mu)$, the set of consistent subbases maximal w.r.t. credibility. Formally, we have

\begin{eqnarray*}
\Wcalbl(B,\mu)  & = & \{ B'_i \subseteq B \mid \bigwedge B'_i\nmodels \neg \mu   \hbox{ and for all  } 
B'_j\subseteq B, j\neq i 
\\ \  &  & \hbox{ s. t. } Bel(B'_i) < Bel(B'_j), \bigwedge B'_j\models \neg \mu\}.
\end{eqnarray*}

Let us turn to explain the computation of the credibility degree of each maximal consistent subbase $Bel(B'_i)$.
As presented below, the CSRG and CSRW operators work in three major steps: the definition of BBAs,
the combination and the decision-making.

% %%%%%%%%%%%%%%%%%%%%%%%%%%%%%%%%%%%%%%%%%
% \subsection{Description of the operators} 
% %%%%%%%%%%%%%%%%%%%%%%%%%%%%%%%%%%%%%%%%%
Indeed, starting from $\Wcal(B, \mu)$ (i.e., the set of maximal consistent subbases $B'_i$ with $1\leq i \leq n$), 
the first step consists in representing each $B'_i$ by a simple BBA denoted $m_i$. 
This function takes into account the cardinality of $B'_i$ in order to reflect its importance 
with regard to the other subbases and compared to the initial agent's belief base $B$. 
Formally, this BBA is given for each $B'_i$ as follows:
%%%%%%%%%%%%%%%%%%
% \begin{center}  
\[
      \left\{
      \begin{array}{l}
      m_i(B'_i)=\frac{|B'_i|}{|B|} \\
      m_i(B)=1-\frac{|B'_i|}{|B|} \\
      \end{array}
      \right.
      \]
% \end{center}
%%%%%%%%%%%%%%%%%
where $|B'_i|$ is the cardinality of $B'_i$ and $|B|$ is the cardinality of $B$. 
As one can remark, $m_i(B'_i)$ represents the proportion of formulas belonging to $B'_i$ with regard to $B$.
Note also that $m_i(B)$ is interpreted as an ignorance level that reflects the belief mass which is not assigned to $B'_i$ 
and therefore transferred to $B$. 

In the second step, the BBAs describing the maximal coherent subbases are combined using Dempster's rule. 
The combined BBA is defined as the orthogonal sum of these BBAs. It is given formally by :
\begin{eqnarray*}
m = m_1\oplus...\oplus m_n
\end{eqnarray*}

Since Dempster's rule is a conjunctive operator and the focal elements $m_i$ of each BBA  are $B'_i$ and $B$,
the focal sets $\Foc(B,\mu)$ of $m$ (the combined BBA)  are therefore all the subbases $B'_i$, 
all the sets derived from their possible combinations (non-empty intersections)
denoted by $\Foci(B,\mu)$ and the initial base $B$. 
This is due to the fact that $B$ is a common focal element defined in each $m_i$. 
Thus, it is clear that $B$ plays a central role in the combination since it allows
appearing all the $B'_i$ and their potential intersections. 
The sets $\Foc(B,\mu)$ and $\Foci(B,\mu)$ are formally defined respectively as follows:
\begin{eqnarray*}
\Foc(B,\mu)=  \Wcal(B, \mu) \cup \Foci(B,\mu) \cup \{B\}
\end{eqnarray*}
\begin{eqnarray*}
\Foci(B,\mu) &=& \{\bigcap_{i,j \in \{1,..n\}} (B'_i,B'_j) \cup \bigcap_{i,j,k \in \{1,..n\}} (B'_i,B'_j,B'_k)
\\ \  &  &  \cup ...  \cup  \bigcap(B'_1,...,B'_n)\} \setminus \{\emptyset\}
\end{eqnarray*}

%%privée del'ensemble !vide OK
%%%%%%%%%%%%%%%%%%%%%%%%%%%%%%%%%%%%%%%%%%%%%%%%%%
%The reason that justifies our choice of the maximum of credibility rule is the following. 
At this point, let us note that the combination allows deducing intersections of subbases with belief masses.
If an intersection (or several) supports completely a subbase $B'_i$, 
it is therefore a focal set affirming $B'_i$. 
Thus, its belief mass can be added to $m(B'_i)$ which allows obtaining an overall 
degree characterizing $B'_i$. This measure is nothing else than its credibility degree $Bel(B'_i)$.
Formally, we have
\begin{eqnarray*}
Bel(B'_i)=m(B'_i)+\sum_{\substack{X\subset B'_i \\ X \in \Foci(B,\mu)}} m(X).
\end{eqnarray*}

%%%%%%%%%%%%%%%%%%%%%%%%%%%%%%%%%%%%%%%%%%%%%%%%%%%%

\medskip
In the last step, the credibility degrees of all the $B'_i$ are exploited for the decision.  
Remind that $\Wcalbl(B,\mu)$ will contain the most credible consistent subbases among all the $B'_i$ of $\Wcal(B,\mu)$.
Hence, the CSRG operator $\revb{\csrg}$ takes into account all the subbases in $\Wcalbl(B,\mu)$ 
considering them equally fair and favorable. This operator can be captured as follows:
\begin{eqnarray*}
B \revb{\csrg}\mu= \bigvee_{B^{'*} \in \Wcalbl(B,\mu)}\bigwedge(B^{'*} \cup \{\mu\})
\end{eqnarray*}

As for the CSRW operator $\revb{\csrw}$, it stems from the intersection of the most credible consistent subbases.
The CSRW operator $\revb{\csrw}$ can be defined as
\begin{eqnarray*}
B \revb{\csrw} \mu= \bigwedge \bigcap_{B^{'*} \in \Wcalbl(B,\mu)} B^{'*} \cup \{\mu\}.
\end{eqnarray*}

Consequently, the associated operators  $\crevb {\csrg}$ $\crevb {\csrw}$ (returning a theory) 
respectively to $\revb{\csrg}$ and $\revb{\csrw}$ are the following:

\begin{eqnarray*}
B \crevb {\csrg} \mu = Cn(\bigwedge \bigvee_{{B^{'*}\in \Wcalbl(B,\mu)}} \bigwedge B^{'*}\cup\{\mu\})
\end{eqnarray*}

\begin{eqnarray*}
B \crevb {\csrw} \mu= Cn(\bigcap_{B^{'*} \in \Wcalbl(B,\mu)} \{B^{'*} \cup \{\mu\} \})
\end{eqnarray*}

We can now define the notion of logical consequence in each of these formalisms.
A formula $\psi$ is a logical consequence of the revision result $B \revb{} \mu$ if:

\begin{itemize}
 \item In the case of CSRG operator, it is a logical consequence of  each subbase $B'_i$ in $\Wcalbl(B, \mu)$ 
 augmented with the new information $\mu$.
Formally, we have

\medskip

\begin{center}
$B\revb{\csrg}\mu \models \psi$ if and only if for each $B'_i \in \Wcalbl(B, \mu), B'_i \cup \{\mu\} \models \psi.$
\end{center} 

\medskip

\item In the case of CSRW operator, it is a logical consequence of the intersection 
of all subbases in $\Wcalbl(B, \mu)$  augmented with the new information $\mu$.
Formally, we have 

\medskip

\begin{center}
$B\revb{\csrw}\mu \models \psi$ if and only if $(\bigwedge \bigcap_{B'_i \in \Wcalbl(B,\mu)}  B'_i) \cup \{\mu\} \models \psi$
\end{center}
\medskip

\end{itemize}

As far, we can conclude that an interpretation $I$ is a model of the revised belief base ($I \models B\revb{}\mu$) 
if and only if $I$ satisfies $\mu$ and satisfies at least one set of $\Wcalbl(B, \mu)$ in the case of CSRG operator, 
and every formula occurring in all maximal consistent subbases (\ie, in all sets of $\Wcalbl(B, \mu)$)
in the case of CSRW operator.

\bigskip

Finally, it is interesting to note that the maximum of credibility is generally 
used within evidence theory to select the most credible element of the frame of discernment.
In this work, we have adapted the use of this rule according to the studied context since the objective
is not to select one formula from the initial set $B$ \footnote{$B$ constitutes the frame of discernment 
and the formulas are the statements.}.
The decision should rather be taken on the different subbases 
(which are subsets of formulas) not on the formulas composing $B$. 
In addition, it is important to notice that we have not hope
to investigate the maximum of plausibility in the context of CSRG and CSRW operators 
since the plausibility considers even the focal elements having a partial intersection with $B'_i$. 
These focal elements are not the results of combining $B'_i$ with other subbases. 
They are rather induced by the combination of other sets.
Similarly, the maximum of pignistic probability 
(basically used to select the most likelihood element of the frame of discernment)
is inappropriate in our context. Indeed, the pignistic transformation cannot be exploited correctly 
in this case since the objective is to select a maximal consistent subbase not a unique formula of $B$.   
%the pignistic probability distribution cannot be exploited correctly to compute a total degree for each subbase.  

%%%%%%%%%%%%%%%%%%%%%%%%%%%%%%%%%%%%%%%%%%%%%%%%%
\subsection{Illustration of CSRG and CSRW operators} 
%%%%%%%%%%%%%%%%%%%%%%%%%%%%%%%%%%%%%%%%%%%%%%%%%
In what follows, we provide a simple example illustrating the CSRG and the CSRW operators.

\begin{example}\label{exp:exp_principal1}  
Consider the belief base  $B=\{ a \rightarrow \neg b, \neg b \rightarrow c, a \rightarrow d, a \rightarrow \neg c,
a \rightarrow \neg d, b, a \rightarrow  b, a \rightarrow  e, \neg b \rightarrow e\}$,
and a new information $\mu=  a \wedge (b \longleftrightarrow e)$. So, we have

\begin{eqnarray*}
  \Wcal(B,\mu)& = & \{B'_1=\{ a \rightarrow \neg b, \neg b \rightarrow c, a \rightarrow d\};
 \\ \  &  & B'_2=\{ a \rightarrow \neg b, a \rightarrow \neg c, a \rightarrow d\};
  \\ \  &  & B'_3=\{ a \rightarrow \neg b, a \rightarrow \neg d\};
 \\ \  &  & B'_4=\{b, a \rightarrow b,  a \rightarrow e, \neg b \rightarrow e\}\}.
\end{eqnarray*}

It is clear that 
$\Wcalc(B,\mu) = \{ B'_4=\{b, a \rightarrow b,  a \rightarrow e, \neg b \rightarrow e\}\}.$
Therefore,
$$B \revb{\rsrg}\mu=b \wedge (a \rightarrow b) \wedge (a \rightarrow e) \wedge (\neg b \rightarrow e)
\wedge (a \wedge (b \longleftrightarrow e))=\perp.$$

As $\Wcalc(B,\mu)$ contains a unique set (\ie $B'_4$), we have
$$\bigcap_{B' \in \Wcalc(B,\mu)} B'=\{\{b, a \rightarrow b,  a \rightarrow e, \neg b \rightarrow e\}\}.$$

Therefore, $B \revb{\rsrw}\mu=B \revb{\rsrg}\mu=\perp.$
Consequently, we obtain

$$B \crevb {\rsrg} \mu = B \crevb {\rsrw} \mu = Cn(\perp).$$
\medskip

Let us now consider the CSRG and CSRW operators. 
Thus, we determine at first the BBA related to each subbase $B'_i$ ($1\leq i \leq 4$) as follows: 
%%%%%%%%%%%%%%%%%%
% \begin{equation*}   
%       \left\{
%       \begin{array}{l}
%       m_1(B'_1)=2/7 \\
%       m_1(B)=5/7 \\
%       \end{array}
%       \right.
% \end{equation*}
% %%%%%%%%%%%%%%%%%
% \begin{equation*}   
%       \left\{
%       \begin{array}{l}
%       m_2(B'_2)=2/7 \\
%       m_2(B)=5/7 \\
%       \end{array}
%       \right.
% \end{equation*}
% %%%%%%%%%%%%%%%%%
% \begin{equation*}   
%       \left\{
%       \begin{array}{l}
%       m_3(B'_3)=2/7 \\
%       m_3(B)=5/7 \\
%       \end{array}
%       \right.
% \end{equation*}
% %%%%%%%%%%%%%%%%%
% \begin{equation*}   
%       \left\{
%       \begin{array}{l}
%       m_4(B'_4)=3/7 \\
%       m_4(B)=4/7 \\
%       \end{array}
%       \right.
% \end{equation*}
% %%%%%%%%%%%%%%%%%
\begin{equation*}   
      \left\{
      \begin{array}{l}
      m_1(B'_1)=3/9 \\
      m_1(B)=6/9 \\
      \end{array}
      \right.
\end{equation*}
%%%%%%%%%%%%%%%%%
\begin{equation*}   
      \left\{
      \begin{array}{l}
      m_2(B'_2)=3/9 \\
      m_2(B)=6/9 \\
      \end{array}
      \right.
\end{equation*}
%%%%%%%%%%%%%%%%%
\begin{equation*}   
      \left\{
      \begin{array}{l}
      m_3(B'_3)=2/9 \\
      m_3(B)=7/9 \\
      \end{array}
      \right.
\end{equation*}
%%%%%%%%%%%%%%%%%
\begin{equation*}   
      \left\{
      \begin{array}{l}
      m_4(B'_4)=4/9 \\
      m_4(B)=5/9 \\
      \end{array}
      \right.
\end{equation*}
%%%%%%%%%%%%%%%%%
Dempster's rule is then applied to combine $m_1$, $m_2$, $m_3$ and $m_4$. 
This leads to the following BBA:   
%%%%%%%%%%%%%%%%%
% \begin{equation*}   
%       \left\{
%       \begin{array}{l}
%       m(B'_1)=0.1145 \\
%       m(B'_2)=0.1145 \\
%       m(B'_3)=0.1145 \\
%       m(B'_4)=0.2147 \\
%       m(\{ a \rightarrow \neg b\})=  m(B'_1 \cap B'_2) +  m(B'_1 \cap B'_3) 
%       \\ \quad \quad \quad \quad \quad \quad +  m(B'_2 \cap B'_3) + m(B'_1 \cap B'_2 \cap B'_3)
%       \\ \quad \quad \quad \quad \quad \quad = 0.1557 \\
%       m(B)=0.2861 \\
%       \end{array}
%       \right.
% \end{equation*}
%%%%%%%%%%%%%%%%%
\begin{equation*}   
      \left\{
      \begin{array}{l}
      m(B'_1)=0.1354 \\
      m(B'_2)=0.1354 \\
      m(B'_3)=0.0774 \\
      m(B'_4)=0.2166\\
      m(\{ a \rightarrow \neg b, a \rightarrow d\})=  m(B'_1 \cap B'_2)= 0.0677 \\
      m(\{ a \rightarrow \neg b\})=  m(B'_1 \cap B'_3) +  m(B'_2 \cap B'_3) 
%       \\ \quad \quad \quad \quad \quad \quad 
       + m(B'_1 \cap B'_2\cap B'_3)
      \\ \quad \quad \quad \quad \quad \quad = 0.0967 \\
      m(B)= 0.2708\\
      \end{array}
      \right.
\end{equation*}

% Based on the combined BBA, we compute the credibility degree $Bel$ for each subbase $B'_i$. 
% Since Dempster's rule allows appearing the focal element $\{ a \rightarrow \neg b\}$
% which is a common formula between $B'_1$, $B'_2$ and $B'_3$, $m(\{ a \rightarrow \neg b\})$ 
% should be considered in the computation of the credibility degrees $Bel(B'_1)$, $Bel(B'_2)$ and $Bel(B'_3)$. 
% As a result, we obtain the following values: 
% %%%%%%%%%%%%%%%%%
% \begin{equation*}   
%       \left\{
%       \begin{array}{l}
%       Bel(B'_1)=m(B'_1)+m(\{ a \rightarrow \neg b\})=0.1145+0.1557=0.2702 \\
%       Bel(B'_2)=m(B'_2)+m(\{ a \rightarrow \neg b\})=0.1145+0.1557=0.2702 \\
%       Bel(B'_3)=m(B'_3)+m(\{ a \rightarrow \neg b\})=0.1145+0.1557=0.2702 \\
%       Bel(B'_4)=m(B'_4)=0.2147 \\
%       \end{array}
%       \right.
% \end{equation*}

Based on the combined BBA, we compute the credibility degree $Bel$ for each subbase $B'_i$. 
Since Dempster's rule allows appearing the focal element $\{  a \rightarrow \neg b, a \rightarrow d\}$ 
which is  common  between $B'_1$ and $B'_2$, $m(\{  a \rightarrow \neg b, a \rightarrow d\})$
should be considered in the computation of the credibility degrees $Bel(B'_1)$ 
and $Bel(B'_2)$. Similarly, since the focal element $\{a \rightarrow \neg b\}$ 
is common  between $B'_1$, $B'_2$ and $B'_3$, $m(\{ a \rightarrow \neg b\})$
should be also added to the credibility degrees $Bel(B'_1)$, $Bel(B'_2)$ and $Bel(B'_3)$.
As a result, we obtain the following values:
%%%%%%%%%%%%%%%%%
\begin{equation*}   
      \left\{
      \begin{array}{l}
      Bel(B'_1)=m(B'_1)+m(\{  a \rightarrow \neg b, a \rightarrow d\})+m(\{ a \rightarrow \neg b\})
       \\ \quad \quad \quad \quad = 0.1354+0.0677+0.0967=0.2998\\
      Bel(B'_2)=m(B'_2)+m(\{ a \rightarrow \neg b, a \rightarrow d\})m(\{ a \rightarrow \neg b\})
       \\ \quad \quad \quad \quad =0.1354+0.0677+0.0967=0.2998 \\
      Bel(B'_3)=m(B'_3)+m(\{ a \rightarrow \neg b\})=0.0774+0.0967=0.1741 \\
      Bel(B'_4)=m(B'_4)=0.2166 \\
      \end{array}
      \right.
\end{equation*}
%%%%%%%%%%%%%%%%%

This yields to have
\begin{eqnarray*}
\Wcalbl(B,\mu) &=& \{ \{B'_1=\{ a \rightarrow \neg b, \neg b \rightarrow c, a \rightarrow d\};
 \\ \  &  & B'_2=\{ a \rightarrow \neg b, a \rightarrow \neg c, a \rightarrow d\}\}.
 \end{eqnarray*}

Observe that  $\Wcalbl(B,\mu) \neq \Wcalc(B,\mu)$
and a fortiori neither $B \revb{\csrg}\mu \neq B \revb{\rsrg}\mu$, nor $B \revb{\csrw}\mu \neq B \revb{\rsrw}\mu$.
As a consequence, 
\begin{eqnarray*}
B \revb{\csrg}\mu & = & (((a \rightarrow \neg b) \wedge (\neg b \rightarrow c) \wedge (a \rightarrow d))
\vee ((a \rightarrow \neg b)\wedge (a \rightarrow \neg c) 
\\ \ &  & \wedge (a \rightarrow d))) \wedge (a \wedge (b \longleftrightarrow e))
\\ \ & = & a \wedge \neg b \wedge d \wedge e.
% 
% 
% \\ \ &  & \vee ((a \rightarrow \neg b) \wedge (\neg d \rightarrow c)))
%  \wedge (b \rightarrow (a \wedge \neg c \wedge \neg d))
%  \\ \ &  = & \neg b \vee (a \wedge \neg c \wedge \neg d)??????
\end{eqnarray*} 
Hence, 
$B \crevb {\csrg} \mu = Cn(a \wedge \neg b \wedge d \wedge e).$ 
Furthermore, we have

$$\bigcap_{B^{'*} \in \Wcalbl(B,\mu)} B^{'*}=\{\{a \rightarrow \neg b, a \rightarrow d\}\}.$$

Therefore,
$B \revb{\csrw}\mu=(a \rightarrow \neg b) \wedge (a \rightarrow d) 
\wedge (a \wedge (b \longleftrightarrow e))=a \wedge d \wedge e.$
Consequently, we obtain 
$B \crevb {\csrw} \mu = Cn(a \wedge d \wedge e).$

\medskip

\medskip

As one can remark, the subbase $B'_4$ was not considered 
in the belief revision outcome using the CSRG operator although 
it has the maximal cardinality with regard to $B'_1$ and $B'_2$.
This is due to the fact that it is less credible than the other subbases. 
This result can be explained by the lack of $(a \rightarrow \neg b)$ and $(a \rightarrow d)$ in $B'_4$ 
which are pertinent formulas in $B$ giving a considerable advantage 
to $B'_1$ and $B'_2$ and involving a loss of credibility of $B'_4$ in favor of these subsets. 
\end{example}

This example illustrates properly the fact that the selection of consistent subbases maximal
with respect to set cardinality is not always a guarantee to select the most 
relevant information since it can neglect potential formulas from the initial 
agent's beliefs and thus it can induce a loss of rationality when revising. 
Additionally, the use of credibility in belief base revision seems reasonable 
complying with the minimality (not in terms of quantity) criterion of belief revision 
through capturing relevant information playing a central role in the initial agent's beliefs.
Hence, this credibility criterion can be interpreted as a way to
guard as possible as far the general sense of the agent beliefs.
So, it would be interesting to study to which extend the credible revision takes the form of syntactic 
revision, implicitly with some semantic features.

%cas ou RSRG et CSRG se comporte de la mm facon

% Let us emphasize that by selecting the most credible consistent subbases, 
% we are giving as revision result different possibilities
% of sets containing formulas that have strong relations 
% 
% 
% Let us note that credibility rate of a consistent subbase in regard to the initial belief base 
% can  give us in some extend a semantic idea about this subbase
% 
% it would be interesting to study whether these postulates hold in the special case6-
\medskip

The example presented above can be a proper illustration of the gap between 
cardinality and credibility as  maximality criteria in the selection of consistent subbases.
But this cannot deny the fact that the CSRG and RSRG operators can lead to the same revision outcome in some cases.
% The following proposition presents a particular case.

\begin{proposition} \label{prop:prop1}
The CSRG operator leads to the same revision result as the RSRG operator 
if and only if the combined BBA verifies the following two conditions:
\begin{itemize}
	\item $\forall B'_i,B'_j\in \Wcalbl(B,\mu)$:
	\begin{equation*}   
      \left\{
      \begin{array}{l}
      m(B'_i)=m(B'_j) \\
      \sum_{\substack{X\subset B'_i \\ X \in \Foci(B,\mu)}} m(X) =
      \sum_{\substack{X\subset B'_j \\ X \in \Foci(B,\mu)}} m(X)\\
      \end{array}
      \right.
\end{equation*}
	\item $\forall B'_i\in \Wcalbl(B,\mu)$ and $\forall B'_j\notin \Wcalbl(B,\mu)$:
	\begin{equation*}   
      m(B'_i)-m(B'_j) > max\left(0, \sum_{\substack{X\subset B'_j \\ X 
      \in \Foci(B,\mu)}} m(X) - \sum_{\substack{X\subset B'_i \\ X \in \Foci(B,\mu)}} m(X)\right)
  \end{equation*}
\end{itemize}
\end{proposition}

\begin{proof} 
As stressed above, the subbases $B'_i$ appear in the combined BBA thanks to the successive intersections with $B$.
Moreover, since $m_i(B'_i)$ is defined with respect to its cardinality $|B'_i|$ (\ie $m_i(B'_i)=\frac{|B'_i|}{|B|}$), 
the successive combinations of $B'_i$ with $B$ using Dempster's rule lead to obtain an order between 
the combined belief masses $m(B'_i)$ that respects the order between the cardinalities $|B'_i|$. As a result: 
\begin{equation*}   
      \left\{
      \begin{array}{l}
      |B'_i|=|B'_j|\Leftrightarrow m(B'_i)=m(B'_j) \\
      |B'_i|>|B'_j|\Leftrightarrow m(B'_i)>m(B'_j)\\
      \end{array}
      \right.
\end{equation*}

If the CSRG and RSRG operators reach the same revision result, this means that $\Wcalbl(B,\mu)=\Wcalc(B,\mu)$. Therefore:
\begin{itemize}
	\item $\forall B'_i,B'_j\in \Wcalbl(B,\mu)$:
	    \begin{equation*}   
      	\left\{
      	\begin{array}{l}
      	|B'_i|=|B'_j| \\
      	Bel(B'_i)=Bel(B'_j) \\ 
      	\end{array}
      	\right.
      \end{equation*}
      %%%%%%%%%%
      \begin{equation*}
       \Leftrightarrow 
      	\left\{
      	\begin{array}{l}
      	m(B'_i)=m(B'_j) \\
      	m(B'_i)+\sum_{\substack{X\subset B'_i \\ X \in \Foci(B,\mu)}} m(X) =  m(B'_j)+
      	\sum_{\substack{X\subset B'_j \\ X \in \Foci(B,\mu)}} m(X) \\
      	\end{array}
      	\right.
      \end{equation*} 
      %%%%%%
      \begin{equation*}
      \Leftrightarrow  
      	\left\{
      	\begin{array}{l}
      	m(B'_i)=m(B'_j) \\
      	\sum_{\substack{X\subset B'_i \\ X \in \Foci(B,\mu)}} m(X) = \sum_{\substack{X\subset B'_j 
      	\\ X \in \Foci(B,\mu)}} m(X)\\
      	\end{array}
      	\right.
      \end{equation*}
%%%%%%%%
	\item $\forall B'_i\in \Wcalbl(B,\mu)$ and $\forall B'_j\notin \Wcalbl(B,\mu)$:
	    \begin{equation*} 
      	\left\{
      	\begin{array}{l}
      	|B'_i|>|B'_j| \\
      	Bel(B'_i)>Bel(B'_j) \\
      	\end{array}
      	\right.
      \end{equation*}
	    %%%%%%%%
	    \begin{equation*} 
	    \Leftrightarrow
      	\left\{
      	\begin{array}{l}
      	m(B'_i)>m(B'_j) \\
      	m(B'_i)+\sum_{\substack{X\subset B'_i \\ X \in \Foci(B,\mu)}} m(X) >  m(B'_j)+
      	\sum_{\substack{X\subset B'_j \\ X \in \Foci(B,\mu)}} m(X) \\
      	\end{array}
      	\right.
      \end{equation*}
      %%%%%%%%
	    \begin{equation*}
	    \Leftrightarrow
      	\left\{
      	\begin{array}{l}
      	m(B'_i)-m(B'_j)>0 \\
      	m(B'_i)-m(B'_j)>\sum_{\substack{X\subset B'_j \\ X \in \Foci(B,\mu)}} m(X) - 
      	\sum_{\substack{X\subset B'_i \\ X \in \Foci(B,\mu)}} m(X)\\
      	\end{array}
      	\right.
      \end{equation*}
      %%%%%%
      \begin{equation*} 
      \Leftrightarrow  
      m(B'_i)-m(B'_j) > max\left(0, \sum_{\substack{X\subset B'_j \\ X \in \Foci(B,\mu)}} m(X) - 
      \sum_{\substack{X\subset B'_i \\ X \in \Foci(B,\mu)}} m(X)\right)
      \end{equation*}
\end{itemize}
\end{proof}

%%%%%%%%
%%%%%%%%
%%%%%%%%
\begin{remark}
Let us emphasize that if the intersection of each pair of subbases is the empty set, 
the CSRG and RSRG operators lead to the same revision result. %(idem for the RSRW and CSRW operators).
Indeed, if it is the case, the intersection of any other group of subbases is also the empty set. 
As a result, the focal elements set of the combined BBA consists only of the subbases (without intersections). 
This implies that there is no intersection supporting totally $B'_i$, i.e.:
\begin{equation*}
\sum_{\substack{X\subset B'_i \\ X \in \Foci(B,\mu)}} m(X)=0
\end{equation*} 
Therefore, $Bel(B'_i)=m(B'_i)$ and as a consequence, the revision result of the CSRG operator can be defined as follows: 
\begin{itemize}
	\item $\forall B'_i,B'_j\in \Wcalbl(B,\mu)$:
	\begin{equation*}   
      Bel(B'_i)=Bel(B'_j)\Leftrightarrow m(B'_i)=m(B'_j)
\end{equation*}
	\item $\forall B'_i\in \Wcalbl(B,\mu)$ and $\forall B'_j\notin \Wcalbl(B,\mu)$:
	\begin{equation*}   
      Bel(B'_i)>Bel(B'_j) \Leftrightarrow m(B'_i)>m(B'_j)\Leftrightarrow m(B'_i)- m(B'_j)>0
  \end{equation*}
\end{itemize}
These two conditions are nothing else than the conditions exposed in Proposition 1.
Hence, the CSRG and RSRG operators lead to the same revision result. 
\end{remark}

%In the following section, we show how to capture beliefs with a mind to be prudent.

%%%%%%%%%%%%%%%%%%%%%%%%%%%%%%%%%%%%%%%%
\section{Compromise revision strategy}
\label{sec:NewRev}
%%%%%%%%%%%%%%%%%%%%%%%%%%%%%%%%%%%%%%%%
In addition to the CSRG and CSRW operators, we present hereunder another contribution 
that explores the idea of using evidence theory in belief base revision.
More specifically, we propose a new revision strategy based on this theory that captures beliefs with a mind to be prudent.

%%%%%%%%%%%%%%%%%%%%%%%%%%%%%%%%%%%%%%%%%%%%%%%%%
\subsection{Description of compromise strategy} 
%%%%%%%%%%%%%%%%%%%%%%%%%%%%%%%%%%%%%%%%%%%%%%%%%
As explained previously, we have already provoked two extreme approaches for revising belief bases. 
The former is permissive and allows choosing all the maximal consistent subbases whereas the latter
is drastic leading to keep only the beliefs that are not questioned. 
Between these two extremes, evidence theory let us shed some light on the idea of an intermediary 
or a \textit{compromise} strategy. The underlying idea of this approach is to capture the maximal 
consistent sets stemming from all the possible intersections of maximal consistent subbases. 
This constitutes obviously an advantage with regard to the drastic strategy which has been intensively
criticized in the literature since it is so prudent and can lead in many cases, 
and particularly with large belief bases, to lose all the initial belief's agents. 
In addition, it presents a benefit for the compromise strategy compared to 
the permissive approach which can lead to an exhaustive revision result especially in the domain of databases repair.

Considering the credibility as the most rational and reliable maximality criterion,
we focus on defining a compromise operator between the CSRG and CSRW ones.
The proposed operator, called CSIR (Credible Sets Intersections Revision), 
is also based on the idea of using the information given by the combined BBA since Dempster's rule 
allows deducing all the possible intersections between the subbases with their related belief masses. 
More precisely, the CSIR operator stems from $\Vcalbl(B,\mu)$, the set of the most credible focal
sets derived from all the possible intersections of maximal consistent subbases $B'_i$. 
In other words, $\Vcalbl(B,\mu)$ selects from $\Foci(B,\mu)$ the sets having the highest credibility degree. 
At this point, let us note that the maximum of credibility implies implicitly 
the satisfaction of maximal set inclusion criterion in $\Vcalbl(B,\mu)$. Formally, we obtain
%It seems natural to wonder whether
\begin{eqnarray*}
\Vcalbl(B,\mu)  & = & \{ X'_i \subseteq \Foci(B,\mu) \mid \bigwedge X'_i\nmodels \neg \mu  
\hbox{ and for all  }  X'_j\subseteq \Foci(B,\mu), j\neq i 
\\ \  &  & \hbox{ s. t. } Bel(X'_i) < Bel(X'_j), \bigwedge X'_j\models \neg \mu\}.
\end{eqnarray*}

% It is from here we start this section in the aim to alleviate the  
The operator $\revb{\csir}$ is therefore defined as follows:

$$B \revb{\csir}\mu= \bigvee_{X^{'*} \in \Vcalbl(B,\mu)}\bigwedge(X^{'*} \cup \{\mu\})$$

and the associated operator  $\crevb {\csrg}$ returning a theory is captured by

$$B \crevb {\csir} \mu = Cn(\bigwedge \bigvee_{{X^{'*}\in \Vcalbl(B,\mu)}} \bigwedge X^{'*}\cup\{\mu\}).$$

A formula $\psi$ is a logical consequence of the revision result $B \revb{\csir} \mu$ if it is
a logical consequence of each set $X'$ in $\Vcalbl(B, \mu)$ augmented with the new information $\mu$. 
Formally, we have

\begin{center}
$B\revb{\csir}\mu \models \psi$ if and only if for each $X' \in \Vcalbl(B, \mu), X' \cup \{\mu\} \models \psi.$
\end{center} 

As far, we can conclude also that in the case of the CSIR operator,
an interpretation $I$ is a model of the revised belief base (\ie $I \models B\revb{}\mu$) 
if and only if $I$ satisfies $\mu$ and satisfies 
every formula occurring in each maximal (\wrt credibility) consistent set derived from 
all subbases intersections (\ie  in each set of $\Vcalbl(B, \mu)$).

Finally, it is worth mentioning that if $B'_i \cap B'_j =\emptyset$ for all $B'_i, B'_j \in \Wcal(B,\mu)$, 
thus $\Foci(B,\mu)=\emptyset$ and therefore $B \revb{\csir}\mu= \mu$.
This obvious result can be explained due to the conflicting character of each pair of subbases. 
Contrary to the drastic approach (Widtio, RSRW and CSRW operators),
this case constitutes the unique situation where the compromise strategy (and particularly the CSIR operator)
loses all the original beliefs. That is why, it would be also interesting to study other compromise operators 
complying with cardinality and set inclusion criteria.

%%%%%%%%%%%%%%%%%%%%%%%%%%%%%%%%%%%%%%%%%%%%%%%%%
\subsection{Illustration of CSIR operator} 
%%%%%%%%%%%%%%%%%%%%%%%%%%%%%%%%%%%%%%%%%%%%%%%%%
We show, in the following example, how compromise strategy and particularly the CSIR operator can be attractive
compared to the drastic strategy.

\begin{example} \label{exp:exp3} 
Let $B$ a belief base such that 
% $B=\{\neg a \vee \neg b,\neg c \vee \neg a,d \vee \neg a,\neg b \vee c
% ,b,\neg a \vee \neg d,\neg d \vee e,\neg c \vee e,\neg b \vee d,\neg d \vee c\}$
%$B=\{\neg a \vee \neg b, b, \neg b \vee c, \neg c \vee \neg a\}$, 
$B=\{a \rightarrow \neg b, c \rightarrow \neg a, \neg d \rightarrow \neg a,  b \rightarrow c,
b,  a \rightarrow \neg d,  d \rightarrow e,  c \rightarrow e,  b \rightarrow d,  d \rightarrow c\}$
and a new information $\mu= a \wedge \neg e$.

\medskip
We have 
\begin{eqnarray*}
\Wcal(B,\mu) & = & \{B'_1=\{a \rightarrow \neg b,  c \rightarrow \neg a, \neg d \rightarrow \neg a,  b \rightarrow c \};
\\  \ &  &  B'_2=\{b,  a \rightarrow \neg d, d \rightarrow e,  c \rightarrow e \}; 
 \\  \ &  & B'_3=\{ a \rightarrow \neg b,  b \rightarrow d,  d \rightarrow c \}
 \\  \ &  & B'_4=\{b, b \rightarrow d,  d \rightarrow c\}\}
%\\ \ & = &\{\{\neg a \vee \neg b, b, \neg b \vee c\}; \{b, \neg b \vee c, \neg c \vee \neg a\}\}
\end{eqnarray*}
and 
\begin{eqnarray*}
\Wcalc(B,\mu) & = & \{B'_1=\{ a \rightarrow \neg b,  c \rightarrow \neg a, \neg d \rightarrow \neg a, b \rightarrow c\};
\\  \ &  &  B'_2=\{b, a \rightarrow \neg d, d \rightarrow e,  c \rightarrow e \}\}.
%\\ \ & = &\{\{\neg a \vee \neg b, b, \neg b \vee c\}; \{b, \neg b \vee c, \neg c \vee \neg a\}\}
\end{eqnarray*}

Let us compute $\Wcalbl(B,\mu)$ and $\Vcalbl(B,\mu)$.
It is obvious that
$$\Foci(B,\mu)   =  \{X'_1= \{  a \rightarrow \neg b\}, X'_2=\{b\}, X'_3=\{b \rightarrow d,  d \rightarrow c\}\}.$$
The BBAs related to all the subbases are at first determined as follows: 
%%%%%%%%%%%%%%%%%%

%%%%%%%%%%%%%%%%%%
\begin{equation*}   
      \left\{
      \begin{array}{l}
      m_1(B'_1)=4/10 \\
      m_1(B)=6/10 \\
      \end{array}
      \right.
\end{equation*}
%%%%%%%%%%%%%%%%%
\begin{equation*}   
      \left\{
      \begin{array}{l}
      m_2(B'_2)=4/10 \\
      m_2(B)=6/10 \\
      \end{array}
      \right.
\end{equation*}
%%%%%%%%%%%%%%%%%
\begin{equation*}   
      \left\{
      \begin{array}{l}
      m_3(B'_3)=3/10 \\
      m_3(B)=7/10 \\
      \end{array}
      \right.
\end{equation*}
%%%%%%%%%%%%%%%%%
\begin{equation*}   
      \left\{
      \begin{array}{l}
      m_4(B'_4)=3/10 \\
      m_4(B)=7/10 \\
      \end{array}
      \right.
\end{equation*}
%%%%%%%%%%%%%%%%%

Dempster's rule is then used to combine these BBAs. The resulting BBA is the following:    
%%%%%%%%%%%%%%%%%
\begin{equation*}   
      \left\{
      \begin{array}{l}
m(B'_1)	=	0.169	\\
m(B'_2)	=	0.169	\\
m(B'_3)	=	0.1086	\\
m(B'_4)	=	0.1086	\\
m(X'_1)	=	0.0724	\\
m(X'_2)	=	0.0724	\\
m(X'_3)	=	0.0466	\\
m(B)	=	0.2534  \\  
      
      \end{array}
      \right.
\end{equation*}
%%%%%%%%%%%%%%%%

Finally, the credibility degrees of the maximal credible consistent subbases and
the credibility degrees of their possible intersections are 
computed based on the combined BBA as below. 
%%%%%%%%%%%%%%%%%
\begin{equation*}   
      \left\{
      \begin{array}{l}
Bel(B'_1)	= m(B'_1) + m(X'_1)= 0.2414\\
Bel(B'_2)	=m(B'_2) + m(X'_2)= 0.2414\\
Bel(B'_3)	= m(B'_3) + m(X'_1) + m(X'_3)= 0.2276\\
Bel(B'_4)	=m(B'_4) + m(X'_2) + m(X'_3)= 0.2276\\
Bel(X'_1)	=m(X'_1)	=	0.0724	\\
Bel(X'_2)	= m(X'_2)	=	0.0724	\\
Bel(X'_3)	= m(X'_3)	=	0.0466	\\      
\end{array}
      \right.
\end{equation*}

This let us to have
\begin{eqnarray*}
\Wcalbl(B,\mu)  & = & \{B'_1=\{ a \rightarrow \neg b, c \rightarrow \neg a, \neg d \rightarrow \neg a,  b \rightarrow c\};
\\  \ &  &  B'_2=\{b, a \rightarrow \neg d,  d \rightarrow e, c \rightarrow e \} \}.
\end{eqnarray*}

Once again, remark that although the set $X'_3$ has the highest cardinality in comparison with $X'_1$ and $X'_2$,
the above computation yields to have
\begin{eqnarray*}
\Vcalbl(B,\mu)  & = & \{X'_1= \{ a \rightarrow \neg b\}, X'_2=\{b\}\}.
\end{eqnarray*}

Turn us now to compute revision result with different operators.
Observe that 
$$\bigcap_{B'_i \in \Wcal(B,\mu)} B'_i=\bigcap_{B'_i \in \Wcalc(B,\mu)} B'_i=\bigcap_{B'_i \in \Wcalbl(B,\mu)} B'_i=\emptyset$$
and so that we lose all initial beliefs taking into account the drastic strategy.
Formally, we obtain
$B \revb{\widtio}\mu = B \revb{\rsrw}\mu = B \revb{\csrw}\mu = \mu$
and trivially
$B \crevb{\widtio}\mu = B \crevb{\rsrw}\mu = B \crevb{\csrw}\mu = Cn(\mu).$
Nonetheless, the compromise strategy avoids falling down in a similar result.
$$B \revb{\csir}\mu  =  (( a \rightarrow \neg b) \vee b) \wedge \mu.$$
This yields automatically 
$B \crevb{\csir}\mu  =  Cn((( a \rightarrow \neg b) \vee b) \wedge \mu).$
It can be clearly seen that CSIR operator behaves in a careful way as expected.
% This constitutes of course an advantage to the compromise strategy with regard to
% the drastic approach since it allows avoiding the loss of beliefs.
% It would be also interesting to mention that , the intersection $\{\neg b \vee d,\neg d \vee c\}$ 
% was not considered as the revision outcome although it has the maximal cardinality in comparison with $\{\neg a \vee \neg b\}$ 
% and $\{b\}$. 
% This is due to the fact that $\{\neg b \vee d,\neg d \vee c\}$  is less credible than $\{\neg a \vee \neg b\}$ and $\{b\}$.
% % This result is obvious since $\{\neg a \vee \neg b\}$ and $\{b\}$ are induced 
% % from the combination of a subbase composed of 4 formulas with another having 3 formulas while $\{\neg b \vee d,\neg d \vee c\}$ 
% % is the result of combining two subbases with 3 formulas for each one. 
\end{example}

%%%%%%%%%%%%%%%%% %%%%%%%%%%%%%%%%%%%%%%%%%%%%%%%%%%%%%%%
% \section{Related works } 
% \label{sec:Relatedworks}
%%%%%%%%%%%%%%%%%%%%%%%%%%%%%%%%%%%%%%%%%%%%%%%%%%%%%%%%
% Previous works dedicated to belief base revision^!
% 

%%%%%%%%%%%%%%%%% %%%%%%%%%%%%%%%%%%%%%%%%%%%%%%%%%%%%%%% 
\section{Conclusion}
\label{sec:Conclusion}
%%%%%%%%%%%%%%%%%%%%%%%%%%%%%%%%%%%%%%%%%%%%%%%%%%%%%%%%
This paper contributes to the current line of research in belief change that 
has received considerable attention from the AI, database and philosophy communities.
In these contexts, revision is considered as the well-known belief change operation 
remaining consistency in order to integrate the new
information while modifying the initial beliefs as little as possible.
As far as we know,  revision strategies in evidence theory have seldom been addressed. 
As shown previously, our work deals with the use of evidence theory in belief base revision that was investigated through
two extreme revision operators (CSRG and CSRW respectively similar to RSRG and RSRW)
and a new compromise revision strategy explicitly instantiated by CSIR operator.
In both contributions, we highlighted the potential benefit offered by the different tools of evidence theory
in belief base revision. Indeed, the notion of BBA was used to model the information related to each subbase. 
Dempster's rule was also applied to combine all these BBAs in order to yield a global BBA synthetizing 
the information given by the subbases set. 
In addition, the credibility maximality was used as a selection criterion to choose the most
credible subbases (CSRG and CSRW) or credible subbases intersections (CSIR) instead of cardinality criterion. 
The presented examples illustrated to which extend the credibility criterion is interesting to avoid losing valuable 
and relavant beliefs.  
% It is interesting to mention 
Let us mention that we can define systematically two other compromise operators 
in the same spirit of CSIR operator:
the former RSIR (Removed sets intersections revision) is an intermediary between RSRG and RSRW operators
respecting the cardinality as maximality criterion  
%of the consistent subbases intersections,
and the latter SIR is an intermediary between Ginsberg and Widtio operators respecting 
the set inclusion as maximality criterion. 
%of the consistent subbases intersections.
% It is well known that belief base revision and non-monotonic inference from
% an inconsistent belief base are the two sides of a same coin [8]. In [4] Cayrol,
% Lagasquie-Schiex and Schiex present a comparative study of some non-monotonic
% syntactic inference relations. Let B be a belief base, < be a total pre-order over
% the formulas of the base, and φ be a propositional formula, the inference rela-
% tions are synthetically defined by (B, <) |∼ p,m φ where p ∈ {T, IN CL, LEX}
% represents the mechanism for selecting consistent subbases, consistent subbases
% maximal w.r.t. set inclusion or w.r.t. to lexicographic order respectively and
% m ∈ {∀, ∃, ARG} represents the inference strategy, universal, existential or ar-
% gumentative respectively.
% Link between our revision operators and their associated in the context of contraction.
% A continuer!!!

Belief revision and belief contraction are two sides of a same coin. In fact, 
unlike belief revision, which allows to incorporate new information into a set of beliefs, 
belief contraction is the process of rationally removing a given belief from a belief set. 
It ensures removing from the set what is necessary to no longer imply this information
while respecting the principle of minimality of change.
Belief contraction is related to belief revision in the sense that belief revision
can be defined in term of contraction, 
%thanks to Levi’s identity [30]. 
that is to say, revising a belief set by new information amounts
to first remove from the belief set any belief contradicting the new information,
and then to add the new information.
Formally, from the works of Levi \cite{Levi77,Levi80} and Harper \cite{Harper77},
the correspondence between belief revision and belief contraction
has been established in \cite{AGM85}, thus providing a useful equivalence
for belief change studies. Since then, different concepts and constructions
have undergone significant elaboration and development (\cite{Gardenfors88,GardenforsMakinson88,Rott93}).
Different formula-based contraction operators have been studied in the literature
% and a set of postulates has been provided for characterizing them [].
that are classified into two families :
the first one includes contraction operators that have been defined in terms of remainder sets,
i.e., maximal subsets of formulas that fail to imply a given formula.
From a dual point of view, the second family represents the Kernel contraction \cite{Hansson94} which
is based on the minimal sub-theories implying the formula by which one contracts. 
Therefore, natural extension of this work is to study the equivalence between formula-based revision operators 
and their corresponding in the context of belief contraction (such as transitively relational 
partial meet contraction, full meet contraction, maxichoice contraction \cite{AGM85} and infra contraction
\cite{BoothMVW11}).
% Initial work [1] in the first family was \textit{partial meet contraction}. The idea of this approach
% of contraction is to choose the intersection of some remainder sets using a selection function.
% Another special case of partial meet contraction, called 
% full meet contraction [1] taking into account only information stemming from the intersection 
% of all remainder sets. 
% Partial meet contraction has been then extended to \textit{Maxichoice contraction} which arises when the
% selection function selects only one remainder set. 
% The result of the contraction can be refined by adding a constraint on the selection function γ.
% In particular, the case where γ is relational and transitive has been considered.
% This approach has been then extended
% to transitively relational partial meet contraction [32]. 
% Alchourrón and Makinson found in [3] that the relationship between the subsets
% maximum inclusion limits that do not imply the information you wish to remove ("re-
% mainder sets") is always maximized. Consequently, the contraction per intersection
% transitive relational partial is a partial intersection contraction generated by
% a transitive and maximized selection function.
% More recently a model-based approach has been proposed for Horn
% contraction and a characterization and a representation theorem have been provided [33].
% we require a slightly more complicated concept
% In turn (a son tour)

Besides, the work of Creignou et al. \cite{CreignouKP17} defines PRSRG and PRSRW operators
as the extension respectively of RSRG  and  RSRW operators to stratified belief bases.
We remind that a stratified belief base $B=(S_1, ..., S_n)$ is provided by a partition of the belief
base in strata $S_i$ ($1 \leq i \leq n$) representing priorities between formulas. 
% which we want to further adapt for our purposes.
It seems so natural to think in the future about the extension of our operators (CRSG, CSRW and CSIR) to
stratified belief bases.
Finally, future work can include also a thorough investigation of the complexity of these new revision operators in
the general case of propositional logic and in some fragments (particularly Horn and Krom fragments).
%An important aspect, Logic programming and databases....
% . We thus want
% to extend our study both to further revision operators and to further fragments of propositional logic, in particular to the
% Horn and Krom case. ...

% However, so far there is not much literature on revision 
% this topic in evidence theory in the case of 

% As far as we know, the problem of belief update within fragments of propositional logic has
% not been addressed so far, except for complexity results in the Horn case
% ????????????????????????????????????????,
% Mentionner le fait que les maximaux par credibilité ou par cardinalité  sont des maximaux par inclusions OK

\bigskip

\bibliographystyle{splncs03}
\bibliography{ontheuse}

\end{document}